\documentclass{article}

\usepackage[utf8]{inputenc} 
\usepackage[T1]{fontenc}    
\usepackage{hyperref}       
\usepackage{url}            
\usepackage{booktabs}       
\usepackage{amsfonts}       
\usepackage{nicefrac}       
\usepackage{microtype}      
\usepackage{xcolor}         
\usepackage{multirow}
\usepackage{graphicx}
\usepackage{subcaption}    
\usepackage{amsmath,amsthm,amssymb}
\usepackage{comment, times,setspace}
\usepackage[linesnumbered,ruled]{algorithm2e}
\usepackage{bm}
\usepackage{float}
\usepackage{enumitem}
\usepackage{url,hyperref} 
\usepackage{xr} 
\usepackage{tikz}
\usepackage{overarrows}
\usetikzlibrary{arrows.meta}
\usepackage{orcidlink}

\usepackage{stmaryrd}
\SetSymbolFont{stmry}{bold}{U}{stmry}{m}{n}
\renewcommand{\thefigure}{\arabic{figure}}

\usepackage{array}

\newtheorem{theorem}{Theorem}

\newtheorem{lemma}{Lemma}
\newtheorem{corollary}{Corollary}

\newtheorem{assumption}{Assumption}

\def\spacingset#1{\renewcommand{\baselinestretch}%
{#1}\small\normalsize} \spacingset{1}

\newcommand{\argmin}{\operatorname*{\arg\min}}

\newcommand{\E}{\operatorname{E}} 

\newcommand{\Var}{\operatorname{Var}}

\newcommand{\bb}{\mathbb}

\usepackage{stmaryrd}

\newcommand{\R}{\bb R}

\newcommand{\Y}{\bm Y}
\newcommand{\y}{\bm y}
\newcommand{\x}{\bm x}
\newcommand{\Z}{\bm Z}
\newcommand{\X}{\bm X}

\newcommand{\U}{\bm U}

\begin{document}
\renewcommand{\thefigure}{\arabic{figure}}

\title{Generative Distribution Prediction: A Unified Approach to Multimodal Learning
\thanks{This work was supported in part by NSF grant DMS-1952539 and NIH grants
R01AG069895, R01AG065636, R01AG074858, U01AG073079. (Corresponding author: Xiaotong Shen.) }}

\author{Xinyu Tian  \thanks{Xinyu Tian is with the School of Statistics, University of Minnesota, MN, 55455 USA (email:
tianx@umn.edu)}, Xiaotong Shen \textsuperscript{\orcidlink{0000-0003-1300-1451}} \thanks{Xiaotong Shen is with the School of Statistics, University of Minnesota, MN, 55455 USA (email: xshen@umn.edu)}.}

\date{}

\maketitle
\begin{abstract}
Accurate prediction with multimodal data—encompassing tabular, textual, and visual inputs or outputs—is fundamental to advancing analytics in diverse application domains. Traditional approaches often struggle to integrate heterogeneous data types while maintaining high predictive accuracy. We introduce Generative Distribution Prediction (GDP), a novel framework that leverages multimodal synthetic data generation—such as conditional diffusion models—to enhance predictive performance across structured and unstructured modalities. GDP is model-agnostic, compatible with any high-fidelity generative model, and supports transfer learning for domain adaptation. We establish a rigorous theoretical foundation for GDP, providing statistical guarantees on its predictive accuracy when using diffusion models as the generative backbone. By estimating the data-generating distribution and adapting to various loss functions for risk minimization, GDP enables accurate point predictions across multimodal settings. We empirically validate GDP on four supervised learning tasks—tabular data prediction, question answering, image captioning, and adaptive quantile regression—demonstrating its versatility and effectiveness across diverse domains.
\end{abstract}

\setlength{\baselineskip}{0.167 in}

\section{Introduction}
The advent of big data and the proliferation of diverse data sources have catalyzed a growing interest in multimodal learning. Multimodal data combines information from multiple sources or modalities. Each modality offers unique insights into the underlying phenomena, and integrating these heterogeneous data types holds the potential for more comprehensive models and enhanced predictive performance \cite{ngiam2011multimodal, baltruvsaitis2019multimodal}.

In contexts such as healthcare, the integration of medical records, imaging, and genomics has shown potential in enhancing outcomes \cite{esteva2017dermatologist}. For applications like autonomous driving, combining data from cameras, lidar, and radar promotes superior decision-making \cite{chen2017multi}. Additionally, in domains such as social media analysis and emotion recognition, merging text, images, and physiological signals offers substantial benefits \cite{wang2015classification, cowie2003describing}. These applications underscore the necessity for models capable of effectively processing multimodal data. However, the integration of heterogeneous modalities poses significant challenges due to variations in structural characteristics, representation, and statistical properties \cite{ramachandram2017deep, baltruvsaitis2019multimodal}. For instance, images are characterized by high dimensionality, text is inherently sequential, and tabular data is structured, which complicates the fusion process.

Traditional methods often utilize prediction models designed for single-modal or tabular data. The complexity of unstructured data, including challenges like high dimensionality and the absence of natural structure in embeddings, further obstructs multimodal learning \cite{Dai2022Embedding}. Thus, these methods find it difficult to effectively manage heterogeneous data types, limiting their applicability in multimodal analysis.

To address these difficulties, we propose \textit{Generative Distribution Prediction} (GDP), a framework that facilitates the integration of multimodal data through conditional synthetic data generation. GDP exploits the core concept that generative models, such as diffusion models
 \cite{sohl2015deep, dhariwal2021diffusion, ho2020denoising, zhang2023text, lin2023diffusion, yuan2023spatio, kotelnikov2023tabddpm, zheng2022diffusion, kim2022stasy} and normalizing flows \cite{kingma2018glow} , capture more expressive features by identifying patterns and dependencies typically overlooked by supervised models, which are often limited to specific distribution characteristics like conditional means and quantiles. The central thesis of this work is that by accurately modeling the data generation process, generative models not only improve the generalizability of supervised learning but also enhance the effectiveness and efficiency of the learning process.

In contrast to conventional models that are focused on specific tasks or modalities, GDP enables the integration of tabular data with unstructured modalities into a unified representation, employing a generative model such as diffusion, augmented with shared embedding mechanisms. By implementing transfer learning, we fine-tune the target model with a pre-trained source model, optimizing embedding sizes to boost accuracy for downstream tasks. This strategy harnesses knowledge from large-scale models trained on extensive datasets while being tailored to the specific requirements of multimodal integration.

GDP presents several distinct advantages. First, it facilitates multimodal data prediction within a single generative modeling framework, enabling the generation of synthetic data that mirrors the original distribution while providing point predictions through risk minimization, adaptable to any loss function. Second, it accommodates mixed data types for supervised tasks by modeling the conditional distribution of variables of interest. This enables GDP to capture complex relationships across different modalities, significantly improving predictive performance and generalization. Third, it enhances robustness and generalizability by incorporating beneficial variability through synthetic data generation, empowering the model to handle new data more effectively. By integrating these features, GDP provides an effective solution for multimodal data integration and predictive modeling.

The contributions of this article are as follows:

\begin{enumerate}
\item \textbf{Generative Distribution Prediction for multimodal supervised learning:}
This paper introduces the Generative Distribution Prediction framework, an innovative approach to multimodal supervised learning that leverages advanced generative models—such as diffusion models and large language-image models—for conditional synthetic data generation. Unlike traditional methods that focus on point predictions, GDP generates predictive distributions, offering insights into associated uncertainty and markedly enhancing predictive accuracy. By focusing on the data-generating distribution, GDP outperforms conventional risk minimization techniques and excels in complex tasks like question answering and image captioning. As shown in Section~\ref{numerical}, GDP consistently outperforms standard regression methods, particularly in scenarios involving tabular responses with multimodal predictors, such as domain adaptation and quantile regression. Moreover, GDP proves effective in unstructured response settings by capturing intricate cross-modal relationships, positioning itself as a robust, scalable, and versatile framework for multimodal supervised learning challenges.

\item \textbf{Theoretical foundations:}
This paper establishes a rigorous theoretical foundation for GDP, highlighting its advantages over traditional predictive models. We identify two key performance factors: (1) generation error, measuring the discrepancy between synthetic and original data distributions, and (2) a controllable synthetic sampling error. Our analysis provides theoretical insights and practical strategies to enhance GDP’s accuracy. Leveraging transfer learning, GDP integrates prior knowledge to improve robustness and uncertainty quantification. Additionally, we provide statistical guarantees for GDP’s performance with diffusion models by bounding generation error in multimodal settings. These capabilities make GDP a powerful tool for real-world multimodal learning challenges.

\item \textbf{Advancements in domain adaptation:}
We propose a novel domain adaptation strategy within the GDP framework, employing a dual-level embedding mechanism to bridge source and target domains. GDP minimizes distributional shifts and adapts effectively to target domains while maintaining predictive accuracy. This dual-level embedding captures cross-domain relationships more effectively than traditional single-modal approaches, which often struggle with domain shifts. Experimental results demonstrate GDP’s superior generalization in shifted domains, particularly with multimodal data. As a result, GDP establishes itself as a powerful solution for navigating complex real-world environments and dynamic data scenarios.
\end{enumerate}

The rest of the paper is structured as follows: Section~\ref{GDP} elaborates on the GDP framework for multimodal supervised learning. Section~\ref{sec: theory} presents theoretical error bounds for the accuracy of point predictions. Section~\ref{sec: gl} explores multimodal learning and domain adaptation using dual-level shared embeddings, facilitating transfer learning between source and target tasks. Building on this foundation, we develop multimodal diffusion models to illustrate the benefits of transfer learning. Section~\ref{numerical} compares the performance of GDP against state-of-the-art supervised methods across various tasks, including tabular regression with domain adaptation, question answering, image captioning, and adaptive quantile regression. Finally, Section~\ref{discussion} examines the broader implications of our methodology and concludes the paper. 

The appendix provides technical details, with Section~\ref{sec: proofs} presenting theoretical proofs, Section~\ref{sec:experiment} outlining the numerical experiment settings in detail, and Section~\ref{sec: diffusion} exploring diffusion models for conditional generation with transferred information, as well as analyzing generation error for statistical guarantees.

\section{Generative distribution prediction}
\label{sec: GDP}

In multimodal supervised learning, the objective is to predict the outcomes of response variables (outputs) based on predictors (inputs) for new, unseen data by learning the underlying relationships from paired input-output data. A key challenge in this framework is domain adaptation, where a supervised model trained on the source domain must adapt to perform well in a different but related target domain. For example, a credit scoring model trained on data from a high-risk population adapts to a low-risk population where defaults are rare. This scenario, known as response shift, occurs when the conditional relationship between input features (covariates) and the response variable remains consistent across domains. Still, the distribution of the response variable differs between the source and target domains. Similarly, covariate shift arises when the marginal distribution of the covariates changes between domains, while the conditional distribution of the response given the covariates remains unchanged \cite{sugiyama2007covariate}. Domain adaptation also encompasses the standard case of supervised learning when the data distribution is consistent across both domains.

To predict the outcomes of response variables $\Y_t$ given a new predictor value $\X_t=\x_t^{new}$,
we consider the predictive probability $P_{\y_t|\x_t^{\text{new}}}$, the conditional probability of 
$\Y_t$ given $\X_t=\x_t^{new}$, with subscript $t$ denoting the target task. Here $\Y_t$  
and $\X_t$ can be tabular, unstructured such as text, or both. This probability,
supported by decision theory \cite{Berger1985}, captures the direct relationship between 
$\X_t$ and $\Y_t$, ensuring accurate predictions by accounting for how different values 
of $\X_t$ affect the conditional probability of $\Y_t$.

\subsection{Generative distribution prediction}
\label{GDP}

Generative Distribution Prediction (GDP) generates predictive distributions and corresponding point predictions by synthesizing data to replicate the conditional distribution of the response given the predictors. Although GDP is capable of providing uncertainty measures from these predictive distributions, this article focuses exclusively on point prediction.
The GDP methodology involves two key steps:

\textbf{Step 1: Constructing a conditional generator for domain adaptation.}  
Using transfer learning, we fine-tune a target generative model based on a pre-trained source model and the target training sample 
$D = \{(\x_t^i, \y_t^i)\}_{i=1}^{n}$,
leveraging dual-level shared embeddings (DSE) as discussed in Section~\ref{transfer}. In scenarios where transfer learning is not applicable, the conditional generator may be trained directly on $D$.

\textbf{Step 2: Using synthetic data for point prediction.}  
The conditional generator produces a synthetic sample
$D_{\x_t^{\text{new}}} = \{(\tilde{\y}_t^k, \x_t^{\text{new}})\}_{k=1}^m$,
representing the responses $\tilde{\Y}_t\sim  \hat{P}_{{\y}_t | \x_t}(\tilde{\y}_t |\x_t^{\text{new}})$ corresponding to a given predictor value \(\x_t^{\text{new}}\), where $\hat{P}_{{\y}_t | \x_t}$ represents the generation distribution obtained in Step 1. To obtain a point prediction, we first construct an empirical prediction loss $m^{-1} \sum_{k=1}^m \ell( \bm\theta(\x_t^{\text{new}}),\tilde{\y}_t^k)$,
where the loss function \(\ell\) (for example, the squared loss) evaluates the prediction performance based on the synthetic responses and the target parameter $\bm \theta(\x_t^{\text{new}})$ (for example, the expectation). The point prediction is then determined by minimizing this loss:
\begin{equation}
\label{loss}
\hat{\bm \theta}(\x_t^{\text{new}}) = \argmin_{ \bm\theta(\x_t^{\text{new}}) \sim \Theta(\x_t^{\text{new}})} m^{-1} \sum_{k=1}^m \ell(\bm\theta(\x_t^{\text{new}}),\tilde{\y}_t^k).
\end{equation}
Here, the parameter space $\Theta(\x_t^{\text{new}})$ can be set as the sample space of  $\hat{P}_{\y_t | \x_t}(\tilde{\y}_t |\x_t^{\text{new}})$.

This methodology offers a unified prediction framework applicable to both tabular and unstructured data. Section~\ref{sec: theory} will
provides a statistical guarantee for this approach, while Section \ref{numerical} explores the benefits of the GDP framework for prediction and 
compare it with classical prediction methods that minimize \eqref{loss} using the original sample $D$.

\subsection{Theory: GDP's predictive risk}
\label{sec: theory}

This subsection establishes a theoretical foundation linking generation error in synthetic data replication—or estimation error in the global data distribution—to local point predictions within the GDP methodology. It highlights the essential role of precise distribution estimation, achieved through high-fidelity data generation, in minimizing downstream risk within the GDP framework. By ensuring accurate synthetic data generation, the GDP approach facilitates robust predictions across diverse loss functions through risk minimization, as defined in \eqref{loss}. In essence, effectively estimating the data-generating distribution enhances performance across downstream prediction tasks. For instance, high-fidelity synthetic data replication leads to strong predictive performance under various loss functions, including absolute loss, hinge loss, and squared error loss, demonstrating the flexibility and effectiveness of the proposed approach.

To formalize this, the excessive risk $R(\bm{\theta}_0, \bm{\theta})$ quantifies the risk associated with loss $\ell$ in \eqref{loss}:
\begin{equation}
R(\bm{\theta}_0(\x_t), \bm{\theta}(\x_t)) = \E_{\y_t|\x_t} \left[ \ell(\bm{\theta}(\x_t),\Y_t) - \ell(\bm{\theta}_0(\x_t),\Y_t) \right],
\end{equation}
where $\E_{\y_t|\x_t}$ is taken concerning the true conditional distribution of $\Y_t$ given $\x_t$. The Wasserstein-1 distance is used as a metric to assess the accuracy of global distribution estimation, defined as:
\begin{equation}
W(P, Q) = \inf_{\gamma \in \Gamma(P, Q)} \int \|x - y\| \, d\gamma(x, y),
\end{equation}
where $\|\cdot\|$ is the Euclidean distance, and $\Gamma(P, Q)$ is the set of all joint distributions (couplings) with marginals $P$ and $Q$.

We introduce the following regularity conditions on the loss function $\ell$ to connect global distribution estimation accuracy with local point predictions:

\begin{assumption}[Lipschitz over bounded domain]
\label{A-1}
The loss $\ell$ satisfies the Lipschitz condition:
{\small 
\begin{equation}
\sup_{(\x_t, \y_t)} |\ell(\bm{\theta}_1(\x_t), \y_t) - \ell(\bm{\theta}_2(\x_t), \y_t)| \leq \beta \|\bm{\theta}_1 - \bm{\theta}_2\|,
\end{equation}
}
and $\sup_{\x_t}\sup_{\bm{\theta}(\x_t)\in \Theta(\x_t)} \|\bm{\theta}(\x_t)\|_{\infty} \leq c_b$, where $\|\cdot\|_{\infty}$ is the vector sup-norm, and $c_b$ and $\beta$ are positive constants.
\end{assumption}

\begin{assumption}[Variance and mean]
\label{A-2}
For any $\x_t$ and some constant $c_v > 0$:
{\small 
\begin{equation}
\Var_{\y_t|\x_t}(\ell(\bm{\theta}(\x_t), \Y_t) - \ell(\bm{\theta}_0(\x_t), \Y_t)) \leq c_v R(\bm{\theta}_0(\x_t), \bm{\theta}(\x_t)),
\end{equation}
}
where $\Var_{\y_t|\x_t}$ denotes the conditional variance given $\x_t$.  
\end{assumption}

\begin{theorem}[GDP's excessive risk]
\label{thm_gdi}
Under Assumptions \ref{A-1}--\ref{A-2}, the excessive risk is bounded as follows:
\begin{equation}
\E R(\bm{\theta}_0, \hat{\bm{\theta}}) \leq  \underbrace{c_1\E W(\hat{P}_{\y_t|\x_t}, P_{\y_t|\x_t})}_{\text{Generation error}} +  \underbrace{c_2m^{-\frac{1}{2}} \log m}_{\text{Synthetic sampling error}},
\end{equation}
where $c_1 = 1+\beta$ and $c_2 = 2^{15} \frac{c_v^{1/2}}{d_{\bm{\theta}}}$, with $d_{\bm{\theta}}$ denoting the dimension of $\bm{\theta}$, and
$\E$ denotes the expectation with respect to the randomness. 
Finally, if $\E W(\hat{P}_{\y_t|\x_t}, P_{\y_t|\x_t})\leq \gamma_n$ with $\gamma_n$ 
representing the generation error bound in terms of training sample size $n$,  then 
  $\E R(\bm{\theta}_0, \hat{\bm{\theta}}) \leq c_1 \gamma_n + c_2 m^{-\frac{1}{2}} \log m$. 
As $m \rightarrow \infty$, $\E R(\bm{\theta}_0, \hat{\bm{\theta}})\leq c_1 \gamma_n$.
\end{theorem}

Theorem \ref{thm_gdi} demonstrates that the accuracy of point predictions, derived via risk minimization as defined in \eqref{loss} on a synthetic sample of size $m$, is determined by two principal components: (i) the generation error inherent in distributional estimation, and (ii) a controllable synthetic sampling error that depends on $m$. The generation error is bounded by the Wasserstein distance,  measuring the distributional shift between the synthetic distribution and the data-generating process induced by a specific generator. As $m$ increases, the synthetic sampling error decreases,  ensuring that the GDP's prediction accuracy is bounded by the generator's Wasserstein error.

\noindent{\bf Risk minimization.}  This result underscores the critical role of precise distribution estimation in risk minimization. A well-estimated distribution inherently facilitates effective risk minimization, enabling optimal performance across a wide range of loss functions. This adaptability marks a paradigm shift: prioritizing accurate distribution estimation enhances downstream predictive performance across various tasks. In Section \ref{quantile}, we further illustrate the significance of this approach by demonstrating its application to adaptive
quantile regression, highlighting its capacity to accommodate diverse predictive objectives.

\noindent\textbf{Impact of synthetic sample size $m$.}  
Theorem \ref{thm_gdi} underscores the significance of the synthetic sample size $m$ in tightening the bound on GDP's excess risk. When $m$ is sufficiently large, the synthetic sampling error diminishes, leaving the excess risk of GDP upper-bounded by the Wasserstein generation error. Empirical results in Sections \ref{subsec-img} and \ref{QnA} support this claim, demonstrating that selecting an appropriate $m$ enhances performance compared to the default choice of $m = 1$ while effectively managing computational cost. This finding highlights the importance of striking a careful balance between computational efficiency and prediction accuracy, with $m$ playing a critical role in this trade-off.

\section{Multimodal learning and domain adaptation}
\label{sec: gl}

This section focuses on diffusion models, though the GDP framework broadly applies to other generative models.

\subsection{Generation via transfer learning with dual-level embedding}
\label{transfer}

To enhance target generation, we leverage a source generation task, transferring knowledge from a pre-trained model to improve performance. The transfer process involves aligning the conditional probability distributions of the source $P(\Y_s|\X_s)$ and the target $P(\Y_t|\X_t)$ using the dual-level shared embeddings (DSE) framework. This alignment allows effective transfer even when the distributions of $(\Y_s, \X_s)$ and $(\Y_t, \X_t)$ differ, facilitating efficient generation with limited target domain data.

The concept of dual-level shared embeddings (DSE) between source and target tasks extends the shared embedding 
concept for transfer learning in \cite{tian2024enhancing}.

\noindent\textbf{Dual-level shared embeddings.} 
Consider an encoder-decoder system, where an encoder \( f \) maps a random vector \( \Y_j \) into a latent embedding vector \( \U_j \), potentially of different dimensionality. A decoder \( g \) then reconstructs the original data from the embedding \( \U_j \), ensuring that \( \Y_j = g(\U_j) \) for \( j = s, t \). The encoder and decoder act as transformations between the semantic latent space and the observed data space, which is especially important for unstructured data.  In an image captioning task, the encoder and decoder translate captions to and from a numerical embedding space. Depending on the styles of captions and photos, the semantic representations of the images may vary in distribution.

The DSE model assumes the conditional distribution of \( \U_j \) given \( \X_j \) $ P_{\bm{u}_j|\x_j}(\cdot|\x_j)$ 
to follow a shared structure across tasks through an embedding function \( h \):
\begin{equation}
\label{dse}
P_{\bm{u}_j|\x_j}(\cdot|\x_j) = P_j(\cdot, h(\x_j)), \quad j = s, t.
\end{equation}

The DSE model in \eqref{dse} captures the shared structure between source and target tasks. 
Based on \eqref{dse}, given the pre-estimated $(\hat{f},\hat{g},\hat{h})$ from the source task, we can learn the distribution of $\U_t$ in the latent space. Let $\theta_t$ parametrize the conditional generator and $P_t(\bm u_t,h(\x_t))$,
We estimate the target-specific latent parameter  by minimizing the target generation loss using the target samples $\{(\bm{y}_t^i, \bm{x}_t^i)\}_{i=1}^{n}$:
\begin{equation}
\label{target-latent}
\hat{{\theta}}_t=\arg\min_{  \theta_t \in \Theta_t} 
\sum_{i=1}^{n} \ell_t \left(\hat{f}\left(\bm{y}_t^i\right), \hat{h}(\bm{x}_t^i);  \theta_t \right),
\end{equation}
where $\Theta_t$ is the parameter space for $ \theta_t$, while $\ell_t$ is the generation loss such as the diffusion score matching loss defined in Section \ref{sec: diffusion} of the appendix.

Given the learned parameters $(\hat{ \theta}_t, \hat{h})$, the generator produces synthetic samples $(\tilde{\Y}_t^i, \X_t)_{i=1}^m$ given $\X_t$ via the decoder $\hat{g}$. Specifically, it completes the generation process by computing $\tilde{\Y}_t^i = \hat{g}(\tilde{\U}_t^i)$. Here, we derive $\hat{g}$ from the source task through shared structures and draw the samples $(\tilde{\U}_t^i, \X_t)_{i=1}^m$  from the distribution $P_t(\cdot; \hat{h}(\X_t), \hat{ \theta}_t)$ given $\X_t$. The distribution $P_t(\cdot; \hat{h}(\X_t), \hat{ \theta}_t)$ denotes the generator's output distribution parameterized by $\hat{ \theta}_t$.

\subsection{Unified diffusion generation for tabular and unstructured data}
\label{sec: diff}

This subsection introduces a diffusion model for synthesizing multimodal samples by integrating structured tabular data with unstructured data (e.g., text and images) for conditional generation. While diffusion models have demonstrated remarkable performance in generating unstructured data \cite{ho2020denoising, song2021score} and tabular data \cite{kotelnikov2023tabddpm}- the literature has thus far treated these domains separately. To our knowledge, no prior work has combined these modalities under a unified diffusion framework. Our approach fills this gap by employing a conditional diffusion framework with shared covariate embeddings and shared response encoder-decoder (as discussed in Section \ref{transfer}), facilitating transfer learning and enabling smooth adaptation between source and target domains. Moreover, to effectively merge structured and unstructured data, we incorporate multiple embedding layers within the embedding function, with adjustable embedding sizes optimized for performance across various downstream tasks.

\begin{figure}[htbp]
  \centering
  \includegraphics[width=0.8\textwidth]{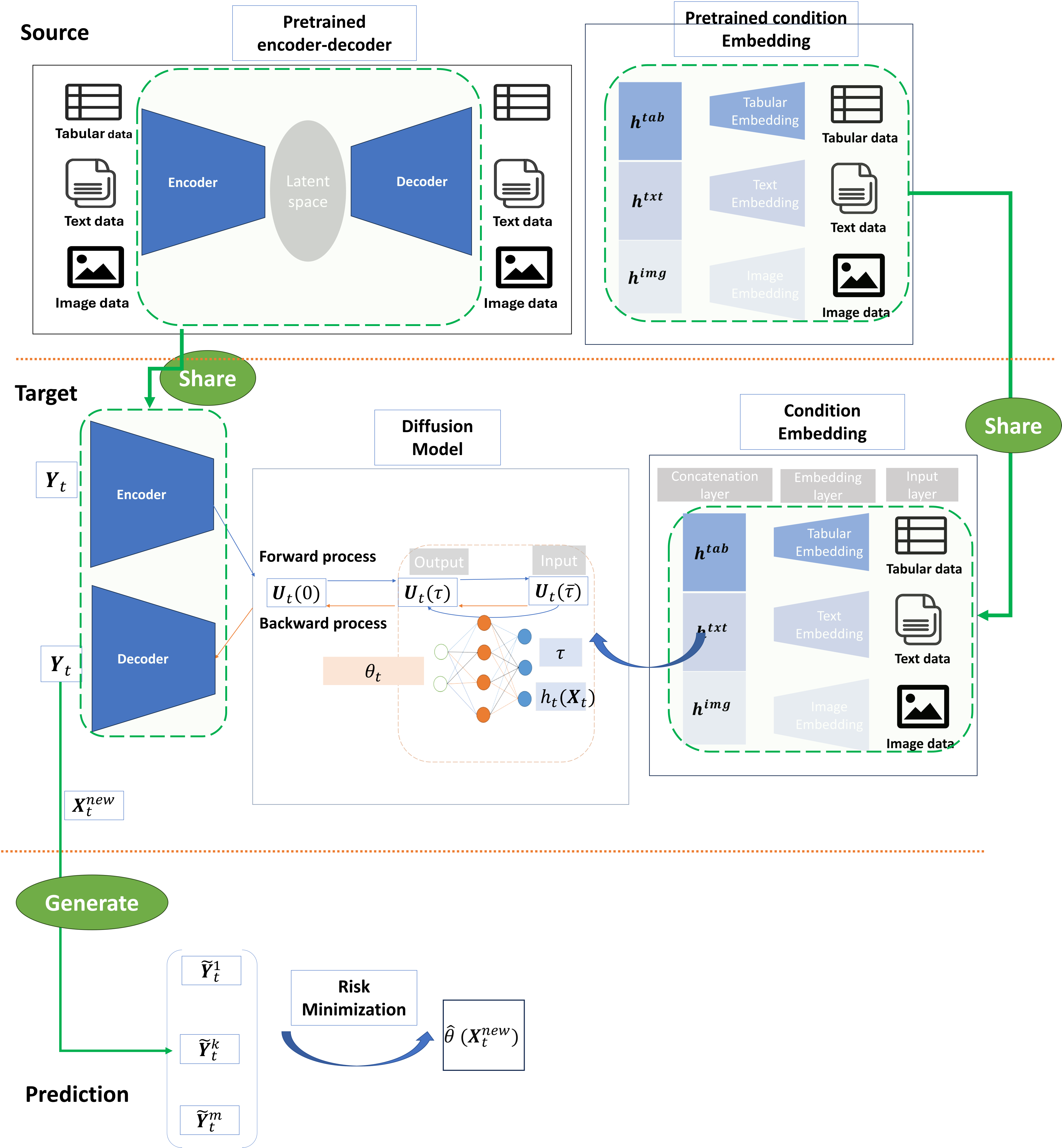}
  \caption{Conditional diffusion models for target domain adaptation through shared representations.} \label{fig_cg}
\end{figure}

Figure \ref{fig_cg} illustrates the model's structure, which enables transfer learning by leveraging both structured (tabular) and unstructured (text and image) data. This architecture supports multimodal supervised learning by converting diverse data types into numerical embeddings via dedicated encoders. For each modality, we implement separate condition embeddings—$h^{tab}$ for tabular data, $h^{txt}$ for text, and $h^{img}$ for image data, then combine them into a unified representation. This integration enhances both multimodal data generation and prediction. Conditional mechanisms provide additional context during the diffusion process, enabling the model to adjust its predictions based on specific input conditions. Meanwhile, the encoder-decoder framework maps response variables into a latent space and then reconstruct the generated representations back into their original formats (tabular, text, or image), completing the generation process.
Key components of the model include shared neural networks $(f, g, h)$ trained on source data. This strategy allows the model to generalize features across domains, improving its capability to handle various data types effectively.

In summary, this diffusion model efficiently supports transfer learning by utilizing shared representations and multimodal embeddings, making it ideal for tasks that require integrating structured and unstructured data across different domains.

\section{Numerical examples}
\label{numerical}

This section rigorously evaluates GDP's performance in tabular and text prediction tasks by assessing its predictive accuracy and benchmarking it against state-of-the-art methods on benchmark datasets and simulated examples. We examine several variants of GDP, including diffusion-GDP (which incorporates diffusion generation with a statistical guarantee), GDP integrated with large pre-trained models, and transfer-GDP enhanced through transfer learning.

Our evaluation encompasses four supervised tasks that yield tabular and text outputs using a combination of tabular, text, and image predictors. Specifically, Subsections \ref{domain-adaptation}–\ref{quantile} detail experiments: (i) domain adaptation using diffusion models for predicting Yelp reviews across source and target domains; (ii) image captioning with diffusion models; (iii) question answering utilizing a large language model; and (iv) adaptive quantile regression across multiple quantiles using diffusion models.

A detailed description of the experiments will be presented in Section \ref{sec:experiment}.

\subsection{Applying diffusion-GDP for domain adaptation in rating prediction}

\label{domain-adaptation}

This subsection evaluates GDP's predictive performance by leveraging a combination of tabular and unstructured predictors and compares it against several popular supervised methods. These include classical approaches such as logistic regression \cite{cox1958logistic}, Naive Bayes \cite{rish2001naive}, Random Forest \cite{breiman2001random}, and SVM \cite{chang2011libsvm}, as well as the large language model BERT \cite{devlin2018bert}. The evaluation considers two scenarios: source and target learning, which rely solely on source and target data, and domain adaptation. While BERT incorporates external data for domain adaptation during source and target learning, GDP enables effective domain adaptation for target learning by leveraging source data.

To assess these models, we use the Yelp reviews dataset (\url{www.yelp.com/dataset}), a widely recognized benchmark for sentiment analysis and star rating prediction. The dataset consists of over 6.9 million reviews rated on a scale from 1 to 5 stars, where a rating of 5 indicates the highest satisfaction. In addition to the review text, the dataset includes business-related attributes such as categories, locations, and user feedback metrics like ``useful", ``funny", and ``cool'' votes. These attributes provide a rich combination of structured and unstructured data for predictive modeling.

For this study, we utilize the first 100,000 text reviews. To create a domain adaptation scenario, the data was partitioned based on the ``cool'' count, which measures how many users deemed a review ``cool.'' Reviews with a ``cool'' count of zero were designated as the source dataset, comprising 79,984 samples and representing non-enthusiastic reviewers. Reviews with a ``cool'' count greater than zero formed the target dataset, consisting of 20,016 samples and representing enthusiastic reviewers.

The distribution of star ratings between non-enthusiastic (source) and enthusiastic (target) two groups, illustrated in Figure~\ref{fig_dy}, reveals a significant difference. Non-enthusiastic reviewers display a more uniform distribution of star ratings, whereas enthusiastic reviewers tend to assign higher ratings. This disparity presents a challenging domain adaptation problem for predictive models. In this context, domain adaptation is essential for effectively leveraging source data to improve performance on the target task, where the source and target data distributions are evidently different but closely related. In both source and target tasks, the datasets are split into training sets and testing sets with ratio 7:3.

\begin{figure}[!htbp]
\label{distr}
  \centering
  \includegraphics[width=0.75\textwidth]{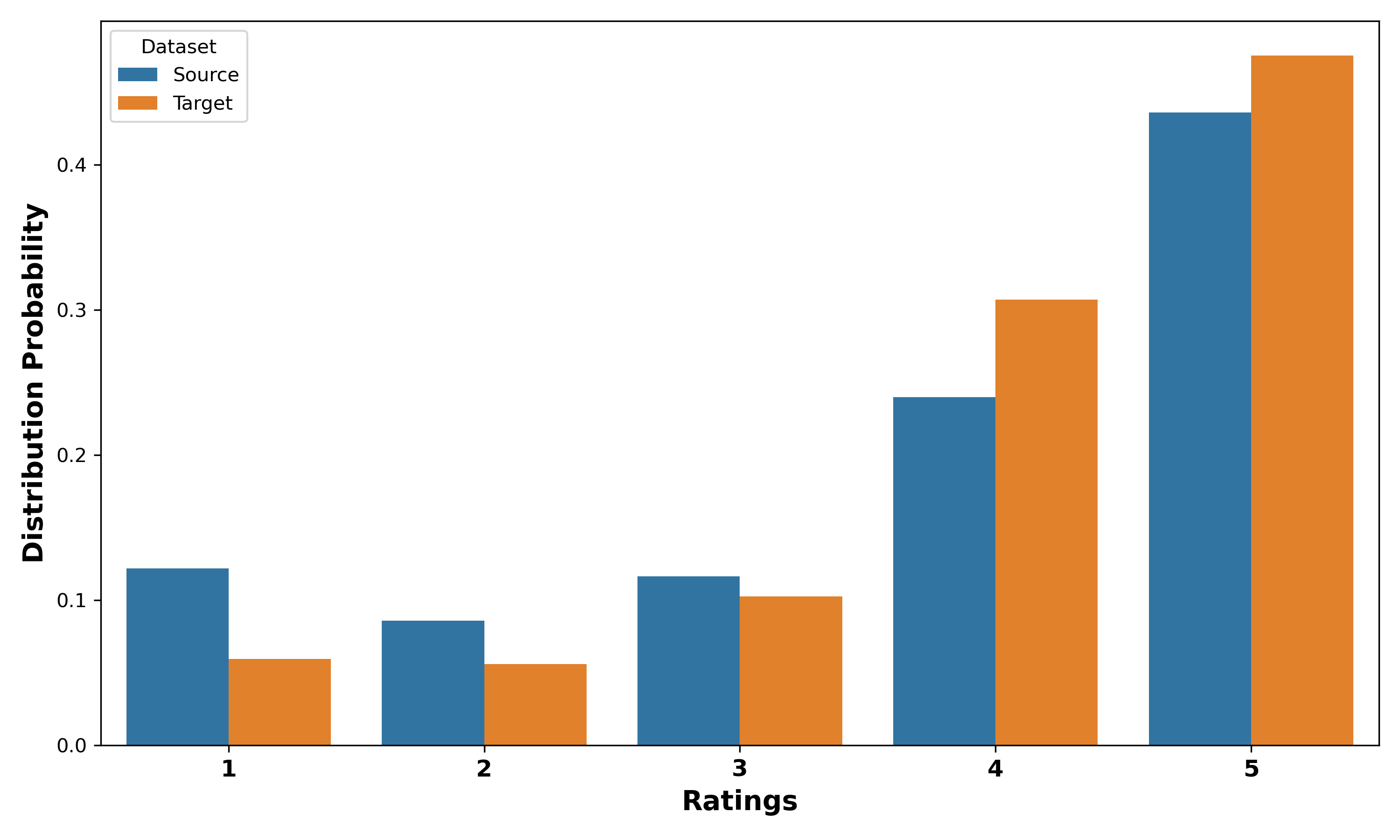}
\caption{Comparison of star rating distributions between non-enthusiastic reviewers (source, blue) and enthusiastic (target, yellow) reviewers.}
\label{fig_dy} 
\end{figure}

 For evalaution, we consider two metrics, classification accuracy and Cohen’s Kappa for both datasets. 
The former refers to the proportion of correctly predicted rating values (1-5), while ``Cohen’s Kappa'' adjusts for the prediction probability agreement by the agreement that could occur purely by chance, which
is particularly useful for evaluating imbalanced situations. 
 
For transfer-GDP, we train a discrete diffusion model \cite{austin2021structured} for categorical response generation given text and use a pre-trained BERT-based embedding model to extract text embeddings from the input for conditional diffusion generation. For GDP, we evaluate synthetic rating modes of samples ($m=50$) conditioned on predictor values using both classification loss and Cohen’s Kappa loss.
For BERT, we fine-tune the model separately for the source and target tasks using source and target data, leveraging the implementation available at \cite{simpletransformers}. For logistic regression, Naive Bayes, Random Forest, and SVM, we use standard machine learning
methods.

\begin{table}[!htbp]
\caption{Model performance on source (non-enthusiastic) and target (enthusiastic) tasks. GDP is tested with (transfer-GDP) and without (nontransfer-GDP) transfer learning, while other methods lack transfer capability. "Accuracy" denotes the percentage of correct labels, and "Cohen's Kappa" measures agreement beyond chance. Best performance is bolded.}
\label{tab:compact_performance}
\setlength{\tabcolsep}{4pt} 
\centering
\begin{tabular}{lcccc}
\toprule
\textbf{Model} & \multicolumn{2}{c}{\textbf{Source}} & \multicolumn{2}{c}{\textbf{Target }} \\
 & \textbf{Accuracy } & \textbf{Kappa} & \textbf{Accuracy} & \textbf{Kappa}   \\
\midrule
Logistic         & 61.52 \% & 0.4685 & 58.88 \% & 0.3779   \\
Naive Bayes      & 58.57 \% & 0.3665 & 54.50 \% & 0.1961   \\
Random Forest    & 53.89 \% & 0.3642 & 51.82 \% & 0.2736   \\
SVM              & 58.80 \% & 0.3732 & 57.05 \% & 0.2914   \\
BERT             & 62.43 \% & 0.4563 & 59.76 \% & 0.3656   \\
Nontransfer-GDP & \textbf{63.21} \%  & \textbf{0.4895}  & 58.10 \% & 0.3670   \\
Transfer-GDP   & &  & \textbf{61.00} \% & \textbf{0.3997}  \\
\bottomrule
\end{tabular}
\end{table}

As shown in Table~\ref{tab:compact_performance}, GDP demonstrated superior performance across both source and target learning tasks, achieving the highest classification accuracy and Cohen's Kappa. BERT ranked second, while traditional methods such as logistic regression, naive Bayes, and random forests performed noticeably worse, highlighting the limitations of classical approaches in this context.

Specifically, for the source dataset, non-transfer GDP improved accuracy by 1.25\% over BERT (from 62.43\% to 63.21\%) and Cohen's Kappa by 7.23\% (from 0.4563 to 0.4895), despite BERT leveraging external data for domain adaptation during source learning. On the target task, GDP, with domain adaptation applied to both source and target data, improved accuracy by 2.07\% over BERT (from 59.76\% to 61.00\%) and Cohen's Kappa by 9.33\% (from 0.3656 to 0.3997), despite BERT’s use of additional external data respectively
for source and target domains. In contrast, GDP without domain adaptation, using only target data, performed slightly worse than BERT for target learning. 

 To compare methods for domain adaptation, we contrast transfer-GDP with nontransfer-GDP for the target task. GDP equipped with domain adaptation via transfer learning showed significant gains over its non-transfer counterpart, improving accuracy by 2.90\% (from 58.10\% to 61.00\%) and Cohen's Kappa by 8.91\% (from 0.3670 to 0.3997). These findings highlight the importance of transfer learning and domain adaptation in addressing learning challenges with limited target data.

The numerical results indicate that GDP achieves higher predictive accuracy, as measured by classification accuracy and agreement metrics, across both source and target domains compared to state-of-the-art competitors. Additionally, its domain adaptation capabilities further enhance predictive accuracy for the target task by leveraging source-domain learning, as evidenced by the superior performance of GDP over alternative methods.
Compared to BERT, which benefits from inherent transfer learning capabilities, GDP demonstrates the added strength of its conditional diffusion models in adapting to varying data distributions. This ability to bridge the gap between non-enthusiastic and enthusiastic reviewers enables GDP to deliver superior domain-adapted predictions. While BERT exhibits strong performance, GDP’s enhancements emphasize the critical role of advanced generative frameworks for tasks involving domain shifts. In contrast, traditional approaches, limited by their inability to adapt to distributional changes, perform significantly worse.

In summary, GDP showcases exceptional adaptability and predictive accuracy, positioning it as a promising solution for complex real-world applications involving heterogeneous and evolving data distributions.

\subsection{Applying GDP with diffusion and large pre-trained models to image captioning}
\label{subsec-img}

This subsection evaluates the performance of the GDP method within a multimodal framework for image captioning. Specifically, we integrate GDP with two distinct multimodal generators—a diffusion model and the Bootstrapping Language-Image Pre-training (BLIP) model \cite{li2022blip}—yielding the diffusion-GDP and BLIP-GDP variants, respectively. Notably, the diffusion-GDP approach is supported by the theoretical results presented in Theorems \ref{thm_gdi} and \ref{thm_ug}, whereas no corresponding theory currently exists for BLIP-GDP.

Modern image captioning typically employs an encoder-decoder framework. A Convolutional Neural Network (CNN)
extracts image features then passes to a decoder—originally an LSTM model \cite{vinyals2015show} 
and more recently a Transformer—to generate captions such as BLIP involving attention mechanisms.

For our experiments, we employ the \texttt{captions\_val2014} subset of the widely recognized COCO Caption dataset (available at \url{http://coco-dataset.org}). We use the first 35,000 samples for training and reserve the remaining 5,504 samples for testing. The COCO Caption dataset is a standard benchmark for image captioning, featuring richly annotated images paired with human-generated captions. Specifically, we train our diffusion model on the designated training set, while we do not fine-tune BLIP since it was pre-trained on the entire COCO Caption dataset. We evaluate both methods on the 5,504 test samples.

In our multimodal diffusion model, as described in Section \ref{sec: diff}, we use SentenceBERT (Sentence-Bidirectional Encoder Representations from Transformers) \cite{reimers2019sentence} to extract sentence embeddings and CLIP (Contrastive Language-Image Pretraining) \cite{radford2021learning} to align textual and visual modalities, together forming the encoder. We then apply a Gaussian diffusion model in the embedding space, leveraging a residual network \cite{he2016deep} to enhance feature representation and maintain model stability. Finally, we use a pre-trained Vec2Text model \cite{morris2024language} to decode the diffused vector embeddings to generate image captions.

To assess alignments between the generated and reference captions, we compute cosine similarity scores between their sentence embeddings, where we obtain these embeddings using models such as Google’s text-embedding-004~\cite{google_text_embedding_004} via APIs. The score ranges from 0 to 1, with 1 indicating complete semantic equivalence. For each image in the test set,  we use the diffusion model to produce $m=10$ candidate captions. The GDP method selects the final caption by minimizing the cosine dissimilarity loss  among 10 candidates, defined as one minus the cosine similarity score, over these candidates, as specified in \eqref{loss}. Finally, we compare the GDP-selected captions with the average performance of all candidate captions from the diffusion model based on their cosine similarity scores with the reference captions.

Figure~\ref{fig-gdp-cap} shows that our multimodal diffusion model (mean: 0.6372, median: 0.6420) performs comparably to the fine-tuned BLIP model (mean: 0.6147, median: 0.6153) when averaging 10 generated captions per image—even though BLIP uses the entire COCO data and additional external data while our diffusion model relies solely on training data. Moreover, after applying GDP selection, the diffusion-GDP and BLIP-GDP methods significantly improve (mean scores of 0.6994 and 0.7099 and median scores of 0.7062 and 0.7203, respectively). This result demonstrates that the GDP selection,  based on risk minimization, consistently chooses captions that better align with the reference, providing a clear advantage over random selection methods.

\begin{figure}[!htbp]
  \centering
  \includegraphics[width=0.75\textwidth]{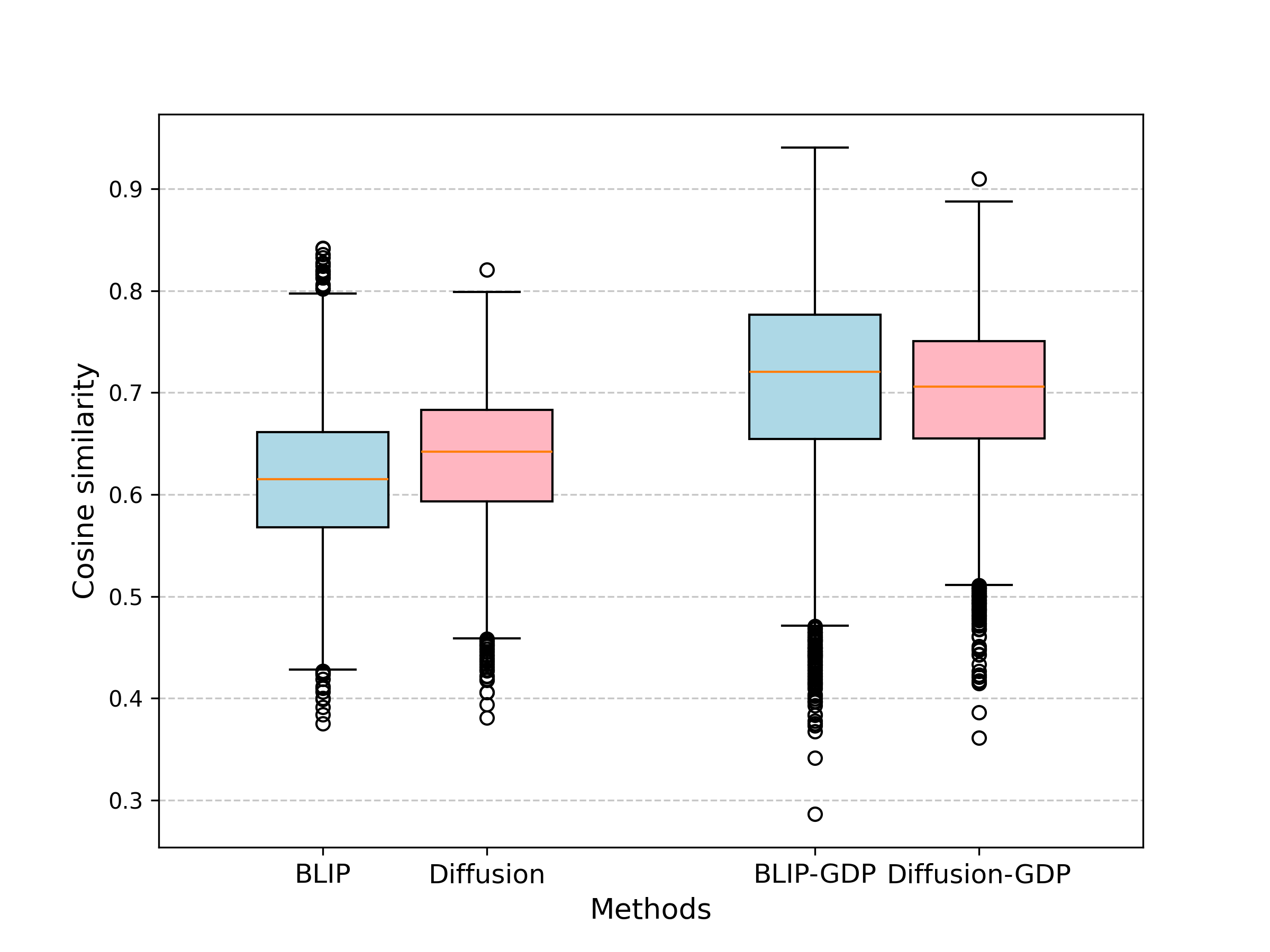}
\caption{Cosine similarity scores for the diffusion method vs. the BLIP model on a test set of 5,054 images. The left panel shows the average similarity between BLIP- or diffusion-generated captions and reference captions across $m=10$ generated captions per image. The right panel shows the average similarity for BLIP-GDP- or diffusion-GDP-selected captions, chosen from the same set of $m=10$ diffusion-generated captions.}
  \label{fig-gdp-cap}
\end{figure}

In Table \ref{tab-image-caption}, an example illustrates how GDP enhances captioning. The GDP method selects captions with high similarity scores generated by diffusion, ensuring better alignment with the image content, while lower scores help filter out low-quality captions that fail to accurately depict the pasta dish in the example image.

\begin{table}[H]
\caption{Example of the GDP method in image captioning. "Similarity (Generation)" represents the average similarity score between a generated caption and nine others from the diffusion model, and "Similarity (Reference)" denotes the average similarity score between a generated caption and five reference captions. The best-selected caption from 10 candidates is
bolded. \label{tab-image-caption}}
\centering
\setlength{\tabcolsep}{1pt} 
\renewcommand{\arraystretch}{1.2}
\begin{tabular}{p{0.25\textwidth}p{0.35\textwidth}>{\centering\arraybackslash}p{0.19\textwidth}>{\centering\arraybackslash}p{0.18\textwidth}}
\hline
\textbf{Image}                                                                                     & \textbf{Diffusion-generated captions}                                                                                                                                   & \textbf{Similarity (Generation)} & \textbf{Similarity (Reference)} \\ \hline
\multicolumn{1}{c}{\multirow{4}{*}{\raisebox{-0.1\height}{\includegraphics[width=3cm,height=3cm]{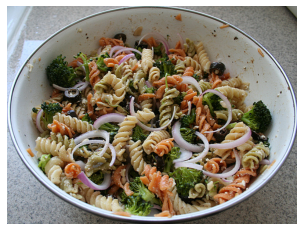}} }}                                                                      & { \scriptsize On a plate you will see a serving of pasta with vegetables on the top}                                                         & \textbf{0.6773}                   & \textbf{0.6675}                    \\
\multicolumn{1}{c}{}                                                                      & { \scriptsize a pasta dish with broccoli, bruzzelli, and serafo sauce on the table. You can see the texture and the slime}                  & 0.6317                             & 0.6094                             \\
\multicolumn{1}{c}{}                                                                      & {\scriptsize A dish of pizza on a plate containing a bowl of vegetables (Bisseta, mash, and cheese) topped.}                                & 0.6291                             & 0.5743                             \\
\multicolumn{1}{c}{}                                           & { \scriptsize A pasta containing a salad with some vegetables (chow and fruit) and some vegetables (snappe and mashi) is sticking}           & 0.6333                             & 0.5968                             \\ \cline{1-1}
\textbf{Reference captions}                                                                        & {\scriptsize An eating platter of grilled cheese is inside of a nasti and a full of sliced grass and a grilled potato.}                     & 0.5811   & 0.5230      \\ \cline{1-1}
{ \scriptsize A large bowl full of noodles and onions}                            & {\scriptsize Pasta with a bowl of noodle on a table. Smoothie pasta and sauce from the pasta bowl have been enjoyed for several months}  & 0.6267                             & 0.6452                             \\
{\scriptsize A bowl filled with pasta and vegetables sitting on a gray counter.} & {\scriptsize Vegetables serving salad dressed with broccoli, cheese, sardonine, pepper, and fruit.  trimmed down}                         & 0.6393                             & 0.6181                             \\
{\scriptsize A bowl of pasta salad with onions and olives.}                      & {\scriptsize An assemblage of a bowl of noodles. A few bowls of pepper, chili, sliced vegetables can be enjoyed for a few minutes.}         & 0.6344                             & 0.6617                             \\
{\scriptsize A corkscrew pasta salad with onions and broccoli.}                  & {\scriptsize A plate of pasta in a pan topped with potatoes and mushrooms. A plate of pasta topped with chilies and parboiled potatoes.} & 0.6463                             & 0.6186                             \\
{\scriptsize a close up of a bowl of food with broccoli noodle salad and onions} & {\scriptsize a sloppy pasta with vegetables in the top, and a salad of broccoli and nasue (cook-nose).}                                     & 0.6223                             & 0.5916                             \\ \hline
\end{tabular} 
\end{table}

\subsection{Applying large language model-GDP to Q\&A}
\label{QnA}

This subsection examines the GDP's performance in question-answering (Q\&A) tasks, where both questions and their corresponding answers are textual. The analysis evaluates the accuracy of responses generated by three methods using the LLaMA-3.1-8B-Instruct model \cite{llama3modelcard} within a Q\&A framework, leveraging the widely recognized WikiQA dataset \cite{yang-etal-2015-wikiqa}. The WikiQA corpus is a benchmark dataset for open-domain Q\&A tasks and consists of 3,047 questions sourced from user queries submitted to Bing, paired with 29,258 candidate answers extracted from Wikipedia sentences. Among these, 1,473 question-answer pairs are labeled as correct, indicating an alignment between the questions and their respective answers.

In this experiment, we randomly sample 100 question-answer pairs from the subset of pairs labeled as correct to serve as the test set. This sampling ensures that the evaluation focuses on instances with human-validated answers based on dataset annotations.

To evaluate the semantic similarity between two text answers, we compute the cosine similarity score between their sentence embeddings, following the same approach described in Section~\ref{subsec-img}.

We evaluate three methods using the LLaMA-3.1-8B-Instruct model \cite{llama3modelcard} to answer questions, with the temperature parameter controlling the randomness or diversity of the model’s outputs. The default temperature is set to 0.7, balancing diversity and determinism. The first method is the GDP approach, which employs a cosine dissimilarity loss (1 minus the cosine similarity score) applied to answers generated by the LLaMA model, as defined in \eqref{loss} with $m=50$. This generative approach selects the answer that minimizes the dissimilarity loss averaged over $m=50$ generated answers for the same question, prioritizing semantic similarity to the correct answer. The second method is the deterministic approach, which uses the LLaMA model with a temperature setting of 0. This configuration ensures that the model always returns the same answer for a given question, emphasizing consistency and accuracy in output. The third method is another generative approach using the default temperature setting of 0.7 without minimizing a loss function over $m=50$ candidates. Unlike the GDP approach, it generates a single answer ($m=1$) for each question, which can vary slightly in semantic meaning across iterations, introducing diversity in the outputs.

As shown in Figure \ref{fig_qa}, the mean similarity scores for the deterministic, generative, and GDP methods on the test set are 0.7670, 0.7590, and 0.7846, respectively, while the corresponding median values are 0.7811, 0.7781, and 0.8118. These results indicate that the GDP method enhances the performance of the generative model.

This systematic comparison underscores the impact of the synthetic sample size $m$ ($m=50$ vs $m = 1$) and the inherent trade-offs between accuracy and diversity. When generating a single answer, the deterministic approach often surpasses the generative approach (using the default temperature) by prioritizing accuracy. In contrast, GDP effectively manages this trade-off by optimizing the final output from a diverse yet comparable set of candidate answers. This approach balances accuracy and diversity, aligning with the discussion on the impacts of synthetic sample size for GDP in Theorem \ref{thm_gdi}.

\begin{figure}[!htbp]
  \centering
  \includegraphics[width=0.75\textwidth]{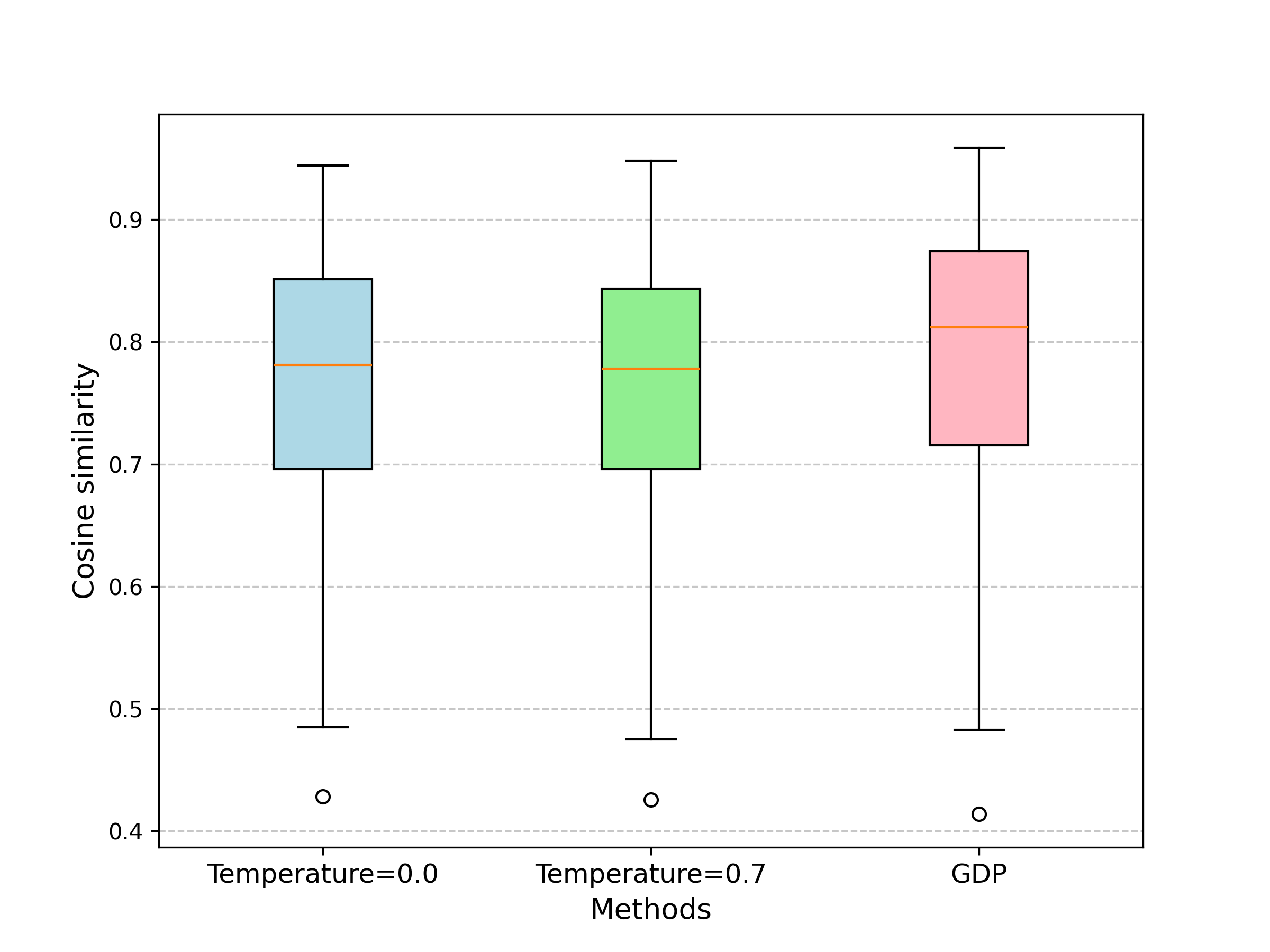}
\caption{Comparative boxplots of cosine similarity scores for three question-answering methods in Q\&A tasks over the test sample of size 100. "Temperature=0.0," "Temperature=0.7," and "GDP" represent the deterministic, generative, and GDP approaches, respectively, based on answers generated by the LLaMA-3.1-8B-Instruct model for a given question.}
\label{fig_qa} 
\end{figure}

\subsection{Diffusion-GDP for adaptive quantile regression}
\label{quantile}

 This experiment considers nonparametric quantile regression via simulations to evaluate the performance of diffusion-GDP for
a prediction task, where the true data-generating mechanism is known. Specifically, let $\X = (X_1, \ldots, X_p)^T$ be a vector of predictors, and define the response $Y$ as 
\begin{eqnarray} \label{model1} 
Y = \sin(\mathbf{X}^T \boldsymbol{\beta}) \;+\; \log\bigl(1 + |X_1|\bigr) \;+\; \epsilon\,(1 + \,|X_2|), 
\end{eqnarray} 
where $\X^T$ denotes the transpose of $\X$, and $\bm \beta$ is the vector of regression coefficients generated uniformly in $(-1,1)$. To induce heteroscedasticity, the noise term $\epsilon$ follows a normal distribution with mean $0$ and variance $|X_2|$, that is, $ \epsilon \sim N\bigl(0, |X_2|\bigr)$.  Hence, the variance is modulated by the magnitude of $|X_2|$. This setup naturally creates a challenging scenario for regression, incorporating both nonlinear and heteroscedastic behavior. Two marginal distributions for \( \X \) are considered: 
\begin{itemize}
    \item \textbf{Case I:} \( \X \sim N(\bm 0, \mathbf{I}) \), where \( \mathbf{I} \) is the identity matrix.
    \item \textbf{Case II:} $\X \sim N(\bm 0, \bm{\Sigma})$, where the covariance matrix 
$\bm \Sigma$ is defined by \( \Sigma_{ij} = \rho^{|i-j|} \) with \( \rho = -0.5 \). Here $\bm{\Sigma}=(\Sigma_{ij})_{p \times p}$ follows an AR(1) dependence structure and $1$ is the common 
variance of each $X_i$ and $\rho$ ($|\rho| <1$) is the lag-1 correlation between two observations $X_i$ and $X_{i+1}$. 
\end{itemize}

In this study, we focus on adapting to multiple quantiles by evaluating the predictive performance of several supervised models across the 5\%, 20\%, 50\%, 80\%, and 95\% quantiles. Although a single supervised model may excel at predicting a specific quantile, it often struggles to maintain robust performance over a range of quantiles. To assess this, we measure out-of-sample prediction errors using two widely adopted metrics: Root Mean Squared Error (RMSE) and Mean Absolute Deviation (MAD).

For diffusion-GDP, we employ a Gaussian diffusion model that offers statistical guarantees (see Theorem \ref{thm_gdi} and Corollary \ref{cor2}). We compare its performance against two state-of-the-art methods—extreme gradient boosting (XGBoost; \cite{chen2016xgboost}) and deep quantile regression (DQR; \cite{padilla2022quantile}). Each model is designed to predict multiple quantiles of $\Y$ given $\x$.

In our simulations, we generate $n = 10{,}000$ observations with $p = 100$ based on the model in~\eqref{model1} for Cases I and II. For quantile regression, GDP trains the diffusion model to perform conditional generation of $\Y$ given $\x$ using the conditional 
probability with $m = 1{,}000$ samples, following the procedure described in Section 2.1. Additionally, XGBoost and DQR optimize using the averaged pinball loss \cite{koenker1978regression, steinwart2011pinball} over the target quantiles of 5\%, 20\%, 50\%, 80\%, and 95\%.

The simulated data are split into training and test datasets with a ratio 7:3. For the testing data, we compute the quantiles as the truth for each $\x$ by sampling $y$ from the true model.
For evaluation, we use Root Mean Squared Error (RMSE) and Mean Absolute Deviation (MAD), computed at each quantile level, and also report the overall performance averaged across the five quantile levels.

As shown in Table~\ref{tab-simu}, GDP achieves lower RMSE and MAD than both XGBoost and DQR across the overall evaluation metrics, except in two situations of Cases I and II in RMSE. This outcome highlights the advantage of more closely approximating the conditional distribution, resulting in tighter alignment with risk-based metrics (e.g., RMSE and MAD), which is consistent with the quantile adaptation result in Corollary \ref{cor1}. 
XGBoost, leveraging its gradient-boosting framework based on decision trees, excels at handling structured data and remains robust to sparse or missing features. Meanwhile, DQR effectively models nonlinearities and heteroscedasticity. However, GDP gains a distinct edge by comprehensively modeling the full conditional distribution, translating into superior performance on metrics sensitive to distributional characteristics. Notably, GDP employs a distribution loss rather than the pinball loss commonly used for quantile regression, enhancing its ability to capture broader distributional nuances.

\begin{table}[!ht]
\caption{RMSE and MAD for quantile regression evaluated at five target quantiles, along with their average performance (denoted by “Average”). The best performance is highlighted in bold.\label{tab-simu}} 
\centering
\begin{tabular}{llllllll}
\hline
\multicolumn{1}{l}{Case I}  &       & 5\%   & 20\%  & 50\%  & 80\%  & 95\%  & Average \\
\hline
                            RMSE  & DQR   & 1.32  & \textbf{0.92}  & \textbf{0.78}  & 0.99  & 1.37  & 1.08    \\
                                  & XGBoost & 1.39  & 0.98  & 0.82 & 1.00 & 1.42  & 1.12 \\
                                  & GDP   & \textbf{1.28} & 0.93 & \textbf{0.78}  & \textbf{0.95} & \textbf{1.31} & \textbf{1.05} \\
\cline{3-8}
                            MAD   & DQR   & 0.93  & \textbf{0.69}  & 0.66 & 0.72 & 0.96  & 0.79    \\
                                  & XGBoost & 0.94  & 0.71 & 0.66  & 0.74  & 0.97  & 0.81    \\
                                  & GDP   & \textbf{0.92} & \textbf{0.69}  & \textbf{0.65}  & \textbf{0.70}  & \textbf{0.95} & \textbf{0.78} \\
\hline
\multicolumn{1}{l}{Case II}        &       & 5\%   & 20\%  & 50\%  & 80\%  & 95\%  & Average \\
\hline
                            RMSE  & DQR   & 1.41  & 0.93  & 0.80  & 1.13  & 1.60  & 1.17    \\
                                  & XGBoost & 1.42 & 0.99  & 0.83 & 1.08 & 1.53 & 1.17 \\
                                  & GDP   & \textbf{1.37} & \textbf{0.90}& \textbf{0.78} & \textbf{0.97}  & \textbf{1.36}  & \textbf{1.08} \\
\cline{3-8}
                            MAD   & DQR   & 0.97  & 0.68 & \textbf{0.66}  & 0.78  & 1.03  & 0.82    \\
                                  & XGBoost & \textbf{0.95}  & 0.74  & \textbf{0.66}  & 0.75 & 0.99  & 0.82    \\
                                & GDP   & 0.97 & \textbf{0.66} & 0.67 & \textbf{0.70}  & \textbf{0.95} & \textbf{0.79} \\
\hline
\end{tabular}
\end{table}

Next, we present a generalized version of Theorem \ref{thm_gdi} for quantile regression. 
\begin{corollary}[Quantile adaptation]
\label{cor1}
Let $\theta$ represent the $\alpha$ quantile of $Y_t$ given $\x_t$. If
we estimate $\theta$ using the pinball loss $$\ell_\alpha(y, \theta(\x_t)) = 
\begin{cases} 
\alpha (y - \theta(\x_t)) & \text{if } y \geq \theta(\x_t), \\
(1 - \alpha)(\theta(\x_t) - y) & \text{if } y < \theta(\x_t),
\end{cases}$$
assuming that the density $p_{y_t|\x_t}(y)$ is lower bounded by $c_l$, then, 
as $m\rightarrow\infty$,
\begin{eqnarray}
\frac{1}{2}c_l\E(\theta_0-\hat\theta)^2\leq \E R(\theta_0,\hat\theta)\leq c_1\E W(\hat{P}_{y_t|\x_t},P_{y_t|\x_t}).
\end{eqnarray}
\end{corollary}

\section{Discussion}
\label{discussion}

The proposed Generative Distribution Prediction (GDP) framework offers a unified approach to multimodal learning by integrating diverse data modalities—tabular, textual, and image—through generative models and shared embeddings. By leveraging transfer learning and conditional generation techniques, GDP directly models the conditional distributions of the response given predictors, enhancing both predictive and generative accuracy. This approach ensures that the generated outputs closely approximate the true data distributions, which improves uncertainty quantification and addresses challenges in domain adaptation and heterogeneous data integration.

Experimental results across various tasks—including Yelp review predictions, image captioning, question answering, and adaptive quantile regression—demonstrate that GDP outperforms traditional models and state-of-the-art methods while exhibiting remarkable robustness and adaptability. Its high generative accuracy is evident in its ability to capture fine-grained data structures and produce outputs that closely mirror real-world distributions, ultimately reducing prediction errors.

Theoretical foundations of the GDP method, established for diffusion models, provide statistical guarantees and scalability for various applications. Notably, recent work \cite{tian2024enhancing} suggests that these theoretical results may be extendable to normalizing flows. However,  establishing statistical guarantees for other generative models such as BLIP remains an open question. This limitation underscores the need for further theoretical exploration to understand GDP's performance benefits across different generative approaches.

Future research could focus on improving computational efficiency and expanding the applicability of GDP, further solidifying its role as a valuable tool for advancing multimodal analytics and deepening our understanding of generative modeling in complex data environments.

GDP’s theoretical underpinnings on diffusion models validate its effectiveness, offering statistical guarantees and scalability for diverse applications. Its versatility underscores significant real-world potential. Future efforts could improve efficiency and applicability, positioning GDP as a valuable tool for advancing multimodal analytics across fields, particularly for statistical theory when GDP 
working with other generative models.

\section*{Appendix}
\label{appendix}
\appendix

\section{Proofs}
\label{sec: proofs}
\begin{proof}[Proof of Theorem \ref{thm_gdi}]
    Let $\tilde{Y}_t$ denote the synthetic sample from $\hat{P}_{\y_t|\x_t}$ and define the corresponding excessive risk based on $\tilde{Y}_t$ as $\tilde{R}(\bm \theta_0,\hat{\bm \theta})$. By Assumption \ref{A-1} and the duality theorem of Kantorovich and Rubinstein \cite{kantorovich1958duality}, the following bounds hold:
\begin{align*}
&\big|R(\bm \theta_0,\bm \theta)-\tilde{R}(\bm \theta_0,\bm \theta)\big| \leq \beta W(\hat{P}_{\y_t|\x_t},P_{\y_t|\x_t}), \\
& \big|\Var_{\y_t|\x_t}(\ell(\bm \theta, \tilde{Y}_t) - \ell(\bm \theta_0, \tilde{Y}_t)) - \Var_{\y_t|\x_t}(\ell(\bm \theta, Y_t) - \ell(\bm \theta_0, Y_t)) \big| \\
&\leq 4c_l\beta W(\hat{P}_{\y_t|\x_t},P_{\y_t|\x_t}).
 \end{align*}

By Assumption \ref{A-2}, 
    \begin{align*}
        \Var_{\tilde{\y}_t|\x_t}(\ell(\bm \theta, \tilde{\Y}_t) - \ell(\bm \theta_0, \tilde{\Y}_t))
 &\leq \Var_{\y_t|\x_t}(\ell(\bm \theta, Y_t) - \ell(\bm \theta_0, Y_t)) + 4c_{\ell}\beta W, \\
        &\leq c_v R(\bm \theta_0,\bm \theta) + 4c_{\ell}\beta W, \\
        &\leq c_v \tilde{R}(\bm \theta_0,\bm \theta) + 2c_{\ell}(c_v + 2\beta) W.
    \end{align*}

    Setting $c_e = 1 + \beta$ yields that
$\{R(\bm \theta_0,\bm \theta) \geq \varepsilon_m + c_e W\} \subset \{\tilde{R}(\bm \theta_0,\bm \theta) \geq \varepsilon_m + W\}$. Hence,
    \begin{align}
       P_{D_{\x_t}}&\big(R(\bm \theta_0,\bm \theta) \geq \varepsilon_m + c_e W\big) 
\leq P_{D_{\x_t}}\big(\tilde{R}(\bm \theta_0,\bm \theta) \geq \varepsilon_m + W\big) \nonumber \\
        &\leq \sum_{j=0}^{\infty} P_{D_{\x_t}}\big(\sup_{A_j} E_m(\tilde{L}_m(\bm \theta) - \tilde{L}_m(\bm \theta_0)) > 2^j (\varepsilon_m + W)\big)  :=\sum_{j=0}^{\infty} I_j,
\label{inq}
    \end{align}
    where $\tilde{L}_m(\bm \theta) = m^{-1} \sum_{k=1}^m \ell(\bm\theta,\tilde{Y}^k_t; \x_t)$ and $A_j = \{\bm\theta: 2^j(\varepsilon_m+W) < \tilde{R}(\bm \theta_0,\bm \theta) \leq 2^{j+1}(\varepsilon_m+W)\}$. 
    
Next, we bound each term $I_j$ ($j=0,\cdots,$) in \eqref{inq} by Lemma \ref{large-d} separately. Toward
this end, we verify \eqref{mean-var} and \eqref{entropy0}. Note that  
\begin{align*}
&\phi(2^j(W+\varepsilon_m),m)\geq \\
 &\frac{m (2^j(\varepsilon_m+W))^{2}}{2[4(c_v2^{j+1}+2c_{\ell}c_v+4c_{\ell}\lambda)(\varepsilon_m+W)+2^j(\varepsilon_m+W) \frac{c_{\ell}}{3} ]}\\
&\geq \frac{m 2^j(\varepsilon_m+W)}{c_k},
\end{align*}
 where $c_k=2[8(c_{v}+c_{v}c_{\ell}+2c_{\ell}\lambda)+\frac{c_{\ell}}{3}]$. This implies
\eqref{mean-var}. Moreover, $\varepsilon_m$ is required to satisfy the entropy inequality,
$$
\int_{r\varepsilon_m^2/8}^{4 c^{1/2}_{v} \varepsilon_m}
H^{1/2}(u, \mathcal{L}) \, du \leq \frac{r^{3/2}}{2^{10}} m^{1/2} \varepsilon_m^2.
$$
where $\mathcal{L}=\{l(\bm \theta,\y),\bm \theta\in [-c_b,c_b]^{d_{\bm \theta}}\}$. Then, $H(u, \mathcal{L})\leq \frac{1}{d_{\bm \theta}}\log(\frac{2B\beta}{u})$. This leads to $\varepsilon_m>c_{\varepsilon}m^{-\frac{1}{2}}\log m$, for some $c_{\varepsilon}$ depending on $c_b$, $c_v$, $r$ and $d_{\bm \theta}$. Specifically, $c_{\varepsilon}$ can be set as $2^{14}\frac{c_v^{1/2}}{d_{\bm \theta}}$ when $m$ is large enough. Then, \eqref{entropy0} holds. Hence,
by Lemma \ref{large-d}, we obtain that
    \begin{align*}
        P_{D_{\x_t}}\big(R(\bm \theta_0,\bm \theta) \geq \varepsilon_m + c_e W\big) \leq 3\exp\left(- (1-r) \frac{m \varepsilon_m}{c_k} \right).
    \end{align*}
    
    To derive the risk bound, let $S = R(\bm \theta_0, \hat{\bm \theta}) - c_e W$. Then,
    \begin{align*}
        \mathbb{E}_{D_{\x_t}} \frac{S}{\varepsilon_m} &\leq 1 + \int_1^{\infty} P(S \geq x\varepsilon_m)\,dx + \frac{1}{\varepsilon_m} P(S > \varepsilon_m) \leq 2.
    \end{align*}
    
    Hence, 
    \begin{align*}
        \mathbb{E}_{D_{\x_t}} R(\bm \theta_0, \hat{\bm \theta}) \leq c_e W + 2c_{\varepsilon} m^{-1/2} \log m.
    \end{align*}
    
By the assumption $\mathbb{E} W(\hat{P}_{\y_t|\x_t}, P_{\y_t|\x_t}) \leq \gamma_n$,  taking expectations on both sides completes the proof.
\end{proof}

\begin{proof} [Proof of Corollary \ref{cor1}:]
    The result is derived directly from the Taylor expansion of the risk,
    \begin{align}
    \label{eq-taylor}
        R(\theta_0,\hat\theta)=\E_{y_t|\x_t} (\ell_{\alpha}(\hat\theta,y_t)-\ell_{\alpha}(\theta_0,y_t))= \frac{1}{2}p_{y_t|\x_t}(\bar\theta)(\hat\theta-\theta_0)^2,
    \end{align}
    where $\bar\theta$ is a point between $\hat\theta$ and $\theta_0$.
    This second-order expansion holds with the fact that the first derivative of the expectation loss with respect to $\theta$ at $\theta_0$ is zero. 

Then we validate the conditions of Theorem \ref{thm_gdi}.  
Note that 
$|\ell_{\alpha}(\theta,y_t)-\ell_{\alpha}(\theta_0,y_t)|\leq \max(\alpha,1-\alpha)|\theta-\theta_0|$.
The pinball loss $\ell_{\alpha}$ satisfies Assumption \ref{A-1} with $\beta=\max(\alpha,1-\alpha)$.
\begin{align*}
    &\Var_{y_t|\x_t}(\ell(\theta(\x_t), Y_t) - \ell(\theta_0(\x_t), Y_t))\\
&\leq \E_{y_t|\x_t}(\ell(\theta(\x_t), Y_t) - \ell(\theta_0(\x_t), Y_t))^2\\
    &\leq \max(\alpha,1-\alpha)^2|\theta-\theta_0|^2\\
    &\leq \frac{2}{c_l} \max(\alpha,1-\alpha)^2R(\theta_0,\hat\theta).
\end{align*}
The last inequality holds with the expansion \eqref{eq-taylor} and lower bound of $p_{y_t|\x_t}$. Then Assumption \ref{A-2} is valid with $c_v=\frac{2}{c_l} \max(\alpha,1-\alpha)^2$.
Hence, by Theorem \ref{thm_gdi}, $\frac{1}{2}c_l\E_{\x_t} (\hat\theta-\theta_0)^2\leq \E R(\theta_0,\hat\theta)\leq c_1\E W(\hat{P}_{\y_t|\x_t},P_{\y_t|\x_t})$. This completes
the proof. 
\end{proof}

\begin{lemma}[Theorem 3 in \cite{shen1994convergence}, Lemma 11 in \cite{tian2024enhancing}]
\label{large-d}
Assume that $f(\Z_j) \in \cal F$ satisfies the Bernstein condition with some constant $U$ for an independent sample $\Z_1,\cdots,\Z_m$, $\mathbb{E}|f(\Z_j)-\mathbb{E}f(\Z_j)|^k\leq \frac{1}{2}k!v_j U^{k-2}$; $j=1,\ldots,m$; $k\geq 2$.  Let $\varphi (\varepsilon_m,v,{\cal F})=\frac{m\varepsilon_m^{2}}{2[4v^2+ \varepsilon_m U /3 ]}$, 
where $\frac{1}{m}\sum_{j=1}^m {\rm Var}(f(\Z_j))\leq m^{-1} \sum_{j=1}^m v_j\leq v$. Assume that 
\begin{eqnarray}
\label{mean-var}
\varepsilon_m \leq r  v/4U, 
\end{eqnarray}
with $0<r<1$ and
\begin{eqnarray}
\label{entropy0}
\int_{r \varepsilon_m/8}^{\sqrt{v}} H_B^{1/2}(u,{\cal F}) du \leq \sqrt{n}\varepsilon_m r^{3/2}/2^{10},
\end{eqnarray}
then 
\begin{align*}
  &P^{*}(\sup_{\{f \in {\cal F}\}} n^{-1} \sum_{i=1}^n (f(\Y^i) -\E f(\Y^i)) \geq \varepsilon_m) \\
&\leq  3
\exp(-(1-r) \varphi(\varepsilon_m,v,m)),    
\end{align*}
where $P^{*}$ denotes the outer probability measure for $\Z_1,\cdots,\Z_m$.
\end{lemma}

\section{Experiment details}
\label{sec:experiment}

\subsection{Experiment settings for Section \ref{domain-adaptation}}
For the machine learning methods used in this experiment, including
logistic regression, Naive Bayes, Random Forest, and SVM, we utilize the standard implementation provided by \cite{scikit-learn} to expand to text predictors. The text data is preprocessed using scikit-learn's \texttt{TfidfVectorizer}, which converts the text into numerical features by analyzing words and bigrams, removing stop words, and filtering out rare terms with a minimum document frequency of 5. The machine learning models follow the configurations outlined in \cite{liu2020yelp}.

For the BERT-based classification model, we employ the \texttt{simpletransformers} library \cite{simpletransformers}. The model is initialized with the pre-trained \texttt{bert-base-cased} architecture and fine-tuned for a 5-label classification task. Training is conducted on a GPU with a maximum sequence length of 64 tokens and a batch size of 16. The model is trained for 10 epochs.

For conditional diffusion generation, we use the discrete diffusion model \cite{austin2021structured} for 
categorical response given text review and the pre-trained \texttt{bert-base-uncased} model as the backbone for generating text embeddings from the input as conditional embeddings, with an embedding dimension of 768. This discrete model replaces 
categorical values with alternative values over time according to a transition matrix, where the diffusion network consists of a two-layer ReLU network. The input layer incorporates the timestep, noisy label, and condition, totaling 774 dimensions, while the hidden layer contains 256 neurons. The optimization process employs the AdamW optimizer for both the diffusion model and the BERT-based embedding model, with a learning rate of \(1 \times 10^{-4}\) for each. The cross-entropy loss (\texttt{nn.CrossEntropyLoss}) is used as the loss function. The batch size is set to 512, and the model is trained for 30 epochs. The timesteps in the diffusion model are set as 1000 from $1\times10^{-4}$ to $0.02$.

\subsection{Experiment settings for Section \ref{subsec-img}}
\label{Exp-B2}

In the image-captioning experiment, we employ a conditional diffusion model to encode multimodal data—including images and text data—into a shared latent embedding space for caption generation, as described in Section \ref{sec: diff}.

Our approach deviates from standard implementations in several key ways: (1) it uniquely integrates CLIP and SentenceBERT for multimodal alignment, (2) it employs a residual network for reverse diffusion instead of the more commonly used U-Net
\cite{ronneberger2015unet}, and (3) it utilizes vec2text to decode denoised embeddings into captions. While building on existing principles, this method makes it a distinctive contribution to multimodal text generation.

To represent image-caption pairs in this unified space, we use SentenceBERT and CLIP: SentenceBERT (GTR-T5-Base, \url{huggingface.co/sentence-transformers/gtr-t5-base}) captures rich linguistic semantics from textual descriptions, while CLIP (CLIP-ViT-B-32, \url{huggingface.co/sentence-transformers/clip-ViT-B-32}) extracts aligned image and text embeddings. All embeddings are standardized to a common dimensionality using trainable projection layers for alignment and then fused through concatenation. By combining SentenceBERT and CLIP, this approach enhances caption quality by refining textual representations beyond what CLIP alone can achieve.

Building on these joint embeddings, we apply a Gaussian diffusion process for conditional caption generation, iteratively injecting and refining noise in the text embeddings based on the given image. We model the reverse diffusion process by a residual network, progressively removing noise while conditioning the image, ensuring the generated caption remains visually relevant. Specifically, the diffusion model inputs a noisy text embedding from GTR-T5-Base, a condition embedding from CLIP-ViT-B-32, and a timestep indicator. These inputs are projected into a unified hidden space of 2048 dimensions, with the timestep further encoded into a time embedding through a two-layer MLP, also with a 2048-dimensional hidden layer.

At the core of the reverse diffusion process is a stack of ten residual blocks, each performing a linear transformation, followed by layer normalization and a ReLU activation to ensure stable training and efficient information flow. After iterative refinement through these blocks, a final projection layer maps the processed hidden representation back to the text embedding space. This mapping allows the model to predict the denoised text embedding at each diffusion step, which is then decoded by vec2text (\url{huggingface.co/ielabgroup/vec2text_gtr-base-st_corrector}) to generate the final image caption.

In the training process, the model is trained for 1000 epochs with 
Adam optimizer with a fixed learning rate of $1 \times 10^{-4}$ and batch size 256. The BLIP image-captioning model (\url{huggingface.co/Salesforce/blip-image-captioning-base}) is implemented with a temperature setting of 1.0.

\subsection{Experiment settings for Section \ref{QnA}}

The details have been provided in Section \ref{QnA} and are thus omitted.

\subsection{Experiment settings for Section \ref{quantile}}

For quantile regression, we employ \texttt{XGBoost} with the quantile loss objective. Hyperparameter tuning is conducted using \texttt{GridSearchCV} with 3-fold cross-validation. The search space includes the following parameters: learning rate (0.01, 0.04, 0.1), maximum tree depth (3, 5, 7), and number of estimators (50, 100, 150). The best hyperparameters are selected based on the pinball loss metric. Once tuning is complete, the optimal model is used to generate predictions on the test set.

For the DQR model, we utilize a ReLU network, with 3-fold cross-validation to determine the optimal network width (64, 128) and network depth (3, 5, 10). During training, the Adam optimizer is used with a fixed learning rate of $1 \times 10^{-4}$ for a maximum of 200 epochs. The model is trained with a batch size of 512, and early stopping is applied using the validation set, with a maximum patience of 20 epochs.

For the GDP method, the neural network architecture for the conditional diffusion models consists of two components: a ReLU network for condition embedding and a ReLU network for score matching, both set to the same network width. Additionally, we employ a sinusoidal time embedding to project the timestep into a vector. The diffusion model operates over 1000 timesteps, ranging from $1 \times 10^{-4}$ to $0.02$. Hyperparameter tuning is performed using 3-fold cross-validation, optimizing the condition embedding dimension (64, 128), time embedding dimension (16, 32, 64), network width (64, 128, 256), and network depth (3, 5). The training strategy follows the same procedure as the DQR model.

\section{Diffusion models and generation accuracy theory}
\label{sec: diffusion}

\subsection{Diffusion modeling}

Recent advancements in generative modeling provide compelling solutions to challenges in supervised learning by harnessing their robust data synthesis capabilities. In particular, models such as diffusion models and normalizing flows—each tailored to specific domains—enable the creation of high-quality synthetic data that closely approximates the true underlying data-generating distribution.

In this section, we rigorously derive the concept of diffusion generalization accuracy. In conjunction with Theorem \ref{thm_gdi} presented in Section \ref{sec: theory}, our derivation establishes a statistical guarantee for the performance of diffusion-GDP in the context of diffusion modeling. Furthermore, the multimodal diffusion model discussed in Section \ref{sec: diff} builds upon this framework by integrating diverse data modalities. Then, we employ an encoder-decoder architecture to generate vector embeddings capturing the intrinsic properties of each modality, followed by diffusion modeling on vector embeddings.

We begin our discussion with a Gaussian diffusion model.

\noindent \textbf{Forward process. } The forward process systematically transforms a random vector $\U(0)$ into white noise by progressively
injecting white noise into a differential equation defined with the Ornstein-Uhlenbeck process, leading to diffused distributions from the initial state $\U(0)$:
\begin{equation}
\label{forward}
\mathrm{d}\U(\tau)=-{b}_{\tau} \U(\tau)\mathrm{d}\tau+\sqrt{2{b}_\tau}\mathrm{d}W(\tau),\quad \tau \geq 0,
\end{equation}
where $\U(\tau)$ has a probability density $p_{\bm u(\tau)}$, $\{W(\tau)\}_{\tau\geq 0}$ represents a standard Wiener process and ${b}_t$ is a non-decreasing weight function. Under \eqref{forward}, $\U(\tau)$ given $\U(0)$ follows $N(\mu_{\tau}\U(0),\sigma^2_{\tau}\bm I)$, where $\mu_{\tau}=\exp(-\int_0^\tau {b}_s\mathrm{d} s)$ and $\sigma^2_{\tau}=1-\mu_{\tau}^2$. Here, we set
${b}_s=1$, which results in $\mu_{\tau}=\exp{(-\tau)}$ and $\sigma^2_{\tau}=1-\exp{(-2\tau)}$. Practically, the process terminates at a sufficiently large $\overline{\tau}$, ensuring the distribution of $\U(\tau)$,
a mixture of $\U(0)$ and white noise, resembles the standard Gaussian vector.

\noindent \textbf{Backward process.} Given $\U(\overline{\tau})$ in \eqref{forward}, a backward process is employed for sample generation for $\U(0)$. Assuming \eqref{forward} satisfies certain conditions \cite{anderson1982reverse}, the backward process $\bm V(\tau)=\U(\overline{\tau}-\tau)$, starting with $\U(\overline{\tau})$, is derived as: 
\begin{align}
\label{reverse}
\mathrm{d}\bm V(\tau)={b}_{\overline{\tau}-\tau}(\bm V(\tau)+2\nabla\log p_
{\bm u(\overline{\tau}-\tau)}(\U(\overline{\tau}-\tau))\mathrm{d}\tau +\sqrt{2{b}_{\overline{\tau}-\tau}}\mathrm{d}W(\tau),
\end{align}
where $\nabla\log p_{\x}$ is the score function which represents the gradient of $\log p_{\x}$.

\noindent \textbf{Score matching.} To estimate the unknown score function, we minimize a matching loss between the score and its approximator $\theta$:
$\int_{0}^{\overline{\tau}}\mathrm{E}_{\bm u(\tau)}\|\nabla \log p_{\bm u(\tau)}(\U(\tau))-\theta(\U(\tau),\tau)\|^2\mathrm{d}\tau$,
where $\|\bm u\|=\sqrt{\sum^{d_u}_{j=1}\bm u_j^2}$ is the Euclidean norm, which is equivalent to minimizing the following loss \cite{oko2023diffusion},
\begin{equation}
\label{loss_2}
\small{
\int_{\underline{\tau}}^{\overline{\tau}}\mathrm{E}_{\bm u(0)}\mathrm{E}_{\bm u(\tau)|\bm u(0)}\|\nabla \log p_{\bm u(\tau)|\bm u(0)}(\U(\tau)|\U(0))-\theta(\U(\tau),\tau)\|^2\mathrm{d}\tau,}
\end{equation}
with $\underline{\tau}=0$. In practice, to avoid score explosion due to $\nabla \log p_{\x(\tau)|\x(0)} \rightarrow \infty$ as $\tau\rightarrow 0$ and to ensure training stability, we restrict the integral interval to $\underline{\tau}>0$ \cite{oko2023diffusion,chen2023improved} in the loss function.
Then, both the integral and $\mathrm{E}_{\bm u(\tau)|\bm u(0)}$ can be precisely approximated by sampling $\tau$ from a uniform distribution on $[\underline{\tau},\overline{\tau}]$ and a sample of $\U(0)$ from the conditional distribution of $\U(\tau)$ given $\U(0)$.

\noindent \textbf{Generation.} To generate a random sample of $\bm V(\tau)$, we replace the score $\nabla\log p_{\bm u(\overline{\tau}-\tau)}$ by its estimate $\hat \theta$ in \eqref{reverse} to yield $\bm V(\tau)$ in the backward equation. For implementation, we may utilize a discrete-time approximation of the sampling process, facilitated by numerical methods for solving stochastic differential equations, such as Euler-Maruyama and stochastic Runge-Kutta methods \cite{song2020denoising}.

\noindent \textbf{Neural network. } An $\mathbb{L}$-layer network $\Phi$ is defined by a composite function
$
\Phi(\x)=(\bm{\mathrm{A}}_\mathbb{L}\sigma(\cdot)+\bm{b}_\mathbb{L})\circ\cdots(\bm{\mathrm{A}}_2\sigma(\cdot)+\bm{b}_2)\circ (\bm{\mathrm{A}}_1\x+\bm{b}_1),
$
where $\bm{\mathrm{A}}_i\in \R^{d_{i+1}\times d_i}$ is a weight matrix and $\bm b_i \in \R^{d_{i+1}}$ is the bias of a linear transformation of the $i$-th layer, and $\sigma$ is the ReLU activation function, defined as $\sigma(\x)=\max(\x,0)$.
Then, the parameter space $\Theta$ is set as $\mathrm{NN}(\mathbb{L},\mathbb{W},\mathbb{S},\mathbb{B},\mathbb{E})$ 
 with $\mathbb{L}$ layers, a maximum width of $\mathbb{W}$, effective parameter number $\mathbb{S}$, the sup-norm $\mathbb{B}$, and parameter bound $\mathbb{E}$:
\begin{align}
\label{p-space}
\mathrm{NN}(d_{in},d_{out},\mathbb{L},\mathbb{W},\mathbb{S},\mathbb{B},\mathbb{E})
=&\{ \Phi: d_1=d_{in}, d_{\mathbb{L}+1}=d_{out}, \max_{1\leq i\leq \mathbb{L}}d_i\leq\mathbb{W},  \nonumber\\
&\sum_{i=1}^{\mathbb{L}}(\|\bm{\mathrm{A}}_i\|_0+ \|\bm{b}_i\|_0)\leq \mathbb{S},  
\|\Phi\|_{\infty}\leq\mathbb{B},\nonumber\\ 
&\max_{1\leq i\leq \mathbb{L}}(\|\bm{\mathrm{A}}_i\|_{\infty}, \|\bm{b}_i\|_{\infty})\leq \mathbb{E}
\},
\end{align}
where $\|\cdot\|_{\infty}$ is the maximal magnitude of entries and $\|\cdot\|_0$ is the number of nonzero entries.

To learn the latent distribution of $\U$ given $\X$, we use the conditional diffusion model which leads to the following conditional score matching loss, for $j=s,t$,
{\small 
\begin{align*}
&l_{u,j}(\bm u,\x;\theta_j,h)=\\
&\int_{\underline{\tau}}^{\overline{\tau}}\mathrm{E}_{\bm u(\tau)|\bm u(0)}\|\nabla \log p_{\bm u(\tau)|\bm u(0)}(\U(\tau)|\bm u)-\theta_j(\U(\tau),h(\x),\tau)\|^2\mathrm{d}\tau.
\end{align*}
}
For the reconstruction loss, it can be set as the square $l_2$ loss, $j=s,t$
$$
l_{r,j}(\bm u,\x;g)=\|g(\bm u)-\x\|_2^2.
$$

With the definition of NN class, we can define $\Theta_u$ with neural networks with different parameters:  $\Theta_u=\mathrm{NN}(d_u+d_h+1,d_u,\mathbb{L}_u,\mathbb{W}_u,\mathbb{S}_u,\mathbb{B}_u,\mathbb{E}_u)$.

Before proceeding, we first introduce the definition of smooth class. 
Let $\bm \alpha$ be multi-index with $|\bm \alpha| \leq \lfloor 
r\rfloor$, where $\lfloor r\rfloor$ is the integer part of $r>0$.  A H\"older ball $\mathcal{C}^{r}(\mathcal D,\R^m,B)$ of radius $B$ with the degree of smoothness $r$ from domain $\mathcal D$ to $\R^m$ is defined by: 
{\small 
\begin{align*}
\label{eq-holder}
 \big\{(g_1,\cdots,g_m): \max_{1\leq l\leq m}\big(\max_{|\bm \alpha| \leq \lfloor r\rfloor}\sup_{\substack{\x}} |\partial^{\bm \alpha} g_l(\x)| + \max_{|\bm \alpha|=\lfloor r\rfloor}\sup_{\substack{\x \neq \y}}\frac{|\partial^{\bm \alpha} g_l(\x) - \partial^{\bm \alpha} g_l(\y)|}{\|\x - \y\|^{r-\lfloor r\rfloor}}\big) < B \big\}.  
\end{align*}
}

\subsection{Theory}

In our transfer learning scenarios, we consider a \emph{source task} and a \emph{target task}. The source task (indexed by $s$) is typically associated with a large amount of labeled data and well-trained models, while the target task (indexed by $t$) may have limited labeled data. In this framework, we leverage a pretrained encoder-decoder pair obtained from the source task to process data in the target task. In particular, the encoder $\hat{f}$ maps the multimodal inputs into a latent embedding space, and the decoder $\hat{g}$ reconstructs the original data from the latent space. Within the latent space, a Gaussian diffusion model is applied to refine or generate new embeddings. Finally, the decoder $\hat{g}$ maps these refined embeddings back to the original multimodal space. This approach facilitates efficient processing and generation of complex multimodal data, while benefiting from the information learned in the source task.

Let $(\hat{f},\hat{g})$ denote a pretrained encoder-decoder pair. For each task $j \in \{s,t\}$, define the latent representation as
$\bm{U}_j = \hat{f}(\Y)$, and the corresponding reconstruction by $\Y_j = \hat{g}(\bm{U}_j)$.
Here, $\Y_j$ denote the response variable, respectively, for task $j$. The source task benefits from a large amount of data and has lower reconstruction error, while the target task may experience a higher error due to domain differences or fewer training samples.

\begin{assumption}[Reconstruction error]\label{A-3}
Suppose that the pretrained encoder-decoder $(\hat{f},\hat{g})$ satisfies
\[
\E\bigl[l_{r,t}(\Y_t; \hat{f},\hat{g})\bigr] \leq \varepsilon_s,
\]
where $l_{r,t}$ is an appropriate reconstruction loss function. Here, the error $\varepsilon_s$ is related to the source task performance and is assumed to be small due to abundant source data and thorough pretraining. Furthermore, assume that the decoder $\hat{g}$ is Lipschitz continuous with Lipschitz constant $\lambda_g$, i.e., for any $\bm{u}, \bm{u}'$,
\[
\|\hat{g}(\bm{u}) - \hat{g}(\bm{u}')\| \le \lambda_g \|\bm{u} - \bm{u}'\|.
\]
\end{assumption}

\begin{assumption}[Latent density]\label{A-4}
Assume that the conditional density of the latent representation $\bm{U}_t = \hat{f}(\Y_t)$ given a feature extraction $\hat{h}(\x_t)$ (where $\hat{h}$ is a known function on $\X_t$) $p^0_{\bm{u}_t|\x_t}(\bm{u}_t|\x_t)$ can be written as
\[
p^0_{\bm{u}_t|\x_t}(\bm{u}_t|\x_t) = \exp\Bigl(-\frac{c_1}{2}\|\bm{u}_t\|^2\Bigr) \cdot k_t\bigl(\bm{u}_t,\hat{h}(\x_t)\bigr),
\]
where $c_1>0$ is a constant, and $k_t: \R^{d_u} \times [0,1]^{d_h} \to \R^+$ is a nonnegative function that is bounded away from zero. In addition, assume that $k_t$ belongs to a H\"older ball
$\mathcal{C}^{r_u}(\R^{d_u} \times [0,1]^{d_h}, \R, B_t)$
with the degree of smoothness parameter $r_u>0$ and bound $B_t>0$, where $d_u$ is the dimension of $\bm{u}_t$ and $d_h$ is the dimension of $\hat{h}(\x_t)$. 
\end{assumption}

\begin{theorem}[Conditional diffusion via transfer learning]\label{thm_ug}
Under Assumptions \ref{A-3} and \ref{A-4}, the expected Wasserstein distance between the true conditional distribution
between the true conditional distribution $P^0_{\y_t|\x_t}$ and its synthetic counterpart $\hat{P}_{\y_t|\x_t}$ generated via transfer learning satisfies 
\[
\E\Bigl[W\Bigl(P^0_{\y_t|\x_t}, \hat{P}_{\y_t|\x_t}\Bigr)\Bigr] = O\Bigl(n^{-\frac{r_u}{d_u+d_h+2r_u}}\log^{m_u} n + \varepsilon_s\Bigr),
\]
where $n$ is the number of target task samples, 
$\varepsilon_s$ denotes the reconstruction error on the target task, which is influenced by the performance on the source task, and
$m_u = \max\Bigl\{\frac{19}{2},\frac{r_u}{2}+1\Bigr\}$.
\end{theorem}

Typically, $\varepsilon_s \ll n^{-\frac{r_u}{d_u+d_h+2r_u}}\log^{m_u} n$ is negligible, as it is mitigated by leveraging a large pre-trained dataset.

\begin{proof}
Let $\bar{P}_{\y_t|\x_t}$ denote the distribution of the reconstructed target variable $\bar{\Y}_t = \hat{g}(\bm{U}_t)$, where $\bm{U}_t = \hat{f}(\Y_t)$ is the latent representation. We first decompose the Wasserstein distance as follows:
\[
W\Bigl(P^0_{\y_t|\x_t}, \hat{P}_{\y_t|\x_t}\Bigr)
\leq W\Bigl(\bar{P}_{\y_t|\x_t}, \hat{P}_{\y_t|\x_t}\Bigr)
+ W\Bigl(P^0_{\y_t|\x_t}, \bar{P}_{\y_t|\x_t}\Bigr).
\]
The first term on the right-hand side captures the error due to the diffusion model in the latent space, and the second term reflects the reconstruction error incurred by $\hat{g}$.

By the Kantorovich--Rubinstein duality \cite{kantorovich1958duality} and using the Lipschitz continuity of $\hat{g}$, we obtain
\[
W\Bigl(\bar{P}_{\y_t|\x_t}, \hat{P}_{\y_t|\x_t}\Bigr)
\leq \lambda_g \, W\Bigl(\bar{P}_{\bm{u}_t|\x_t}, \hat{P}_{\bm{u}_t|\x_t}\Bigr),
\]
where $\bar{P}_{\bm{u}_t|\x_t}$ and $\hat{P}_{\bm{u}_t|\x_t}$ denote the true and estimated distributions of the latent variable $\bm{u}_t$, respectively.

Under Assumption \ref{A-4} and by applying Theorem 2 of \cite{tian2024enhancing}, it follows that
\[
\E\Bigl[W\Bigl(\bar{P}_{\bm{u}_t|\x_t}, \hat{P}_{\bm{u}_t|\x_t}\Bigr)\Bigr]
= O\Bigl(n_t^{-\frac{r_u}{d_u+d_h+2r_u}}\log^{m_u} n_t\Bigr).
\]

For the second term, by definition of the reconstruction error, we have
\[
W\Bigl(P^0_{\y_t|\x_t}, \bar{P}_{\y_t|\x_t}\Bigr)
\le \E\Bigl[\|\hat{g}(\hat{f}(\y_t))-\y_t\|\Bigr]
\le \varepsilon_s.
\]

Combining these bounds yields
\[
\E\Bigl[W\Bigl(P^0_{\Y_t|\x_t}, \hat{P}_{\Y_t|\x_t}\Bigr)\Bigr]
\le \lambda_g \, O\Bigl(n_t^{-\frac{r_u}{d_u+d_h+2r_u}}\log^{m_u} n_t\Bigr)
+ \varepsilon_s,
\]
which establishes the claimed result. This completes the proof. 
\end{proof}

Next, we present a result for non-transfer diffusion that does not leverage the source learning.
\begin{corollary} [Non-transfer conditional diffusion generation]
\label{cor2}
Under the conditions of Theorem \ref{thm_ug} with $f$, $g$, $h$ known as identical mappings, 
we have the error bound for the simple case, 
\[
\E\Bigl[W\Bigl(P^0_{\Y_t|\x_t}, \hat{P}_{\Y_t|\x_t}\Bigr)\Bigr] = O\Bigl(n^{-\frac{r_u}{d_y+d_x+2r_u}}\log^{m_u} n\Bigr).
\]
\end{corollary}
This is a direct consequence of Theorem \ref{thm_ug} with $d_u=d_y$, $d_h=d_x$ and $\varepsilon_s=0$.

\bibliographystyle{abbrv}
\bibliography{augmentation,preref-1}

\end{document}